\documentclass{article}
\usepackage{spconf,amsmath,amssymb,dsfont,graphicx,verbatim}
\graphicspath{{figures/}}
\usepackage{amsmath,amssymb,dsfont,verbatim}
\usepackage{import}
\usepackage{subcaption}
\usepackage{url, hyperref}
\usepackage{cite}
\usepackage{xcolor}
\usepackage{booktabs} 
\usepackage{tikz}
\usetikzlibrary{mindmap}

\allowdisplaybreaks[0]

\let\oldbibliography\thebibliography
\renewcommand{\thebibliography}[1]{\oldbibliography{#1}
\setlength{\itemsep}{-2pt}} %Reducing spacing in the bibliography.

\usepackage{amsthm}

\DeclareMathOperator*{\argmin}{arg\,min}

\DeclareMathOperator{\T}{\mathsf{T}}
\DeclareMathOperator{\E}{\mathds{E}}

\DeclareMathOperator{\w}{\boldsymbol{w}}
\DeclareMathOperator{\x}{\boldsymbol{x}}
\DeclareMathOperator{\s}{\boldsymbol{s}}

\DeclareMathOperator{\y}{\boldsymbol{y}}

\usepackage{amsthm}

\theoremstyle{plain}
\newtheorem{assumption}{Assumption}
\newtheorem{theorem}{Theorem}

\newtheorem{lemma}{Lemma}

\theoremstyle{definition}

\newtheorem{example}{Example}

\title{Second-Order Guarantees in Federated Learning}
\name{Stefan Vlaski, Elsa Rizk and Ali H. Sayed\thanks{Emails:\{stefan.vlaski, elsa.rizk, ali.sayed\}@epfl.ch. Preliminary results limited to single, unbiased local updates (\( E_k = 1 \)) appear in~\cite{Vlaski20Kailath}.}
\address{School of Engineering, \'{E}cole Polytechnique F\'{e}d\'{e}rale de Lausanne}}
\begin{document}
\ninept
\maketitle
\begin{abstract}
 Federated learning is a useful framework for centralized learning from distributed data under practical considerations of heterogeneity, asynchrony, and privacy. Federated architectures are frequently deployed in deep learning settings, which generally give rise to non-convex optimization problems. Nevertheless, most existing analysis are either limited to convex loss functions, or only establish first-order stationarity, despite the fact that saddle-points, which are first-order stationary, are known to pose bottlenecks in deep learning. We draw on recent results on the second-order optimality of stochastic gradient algorithms in centralized and decentralized settings, and establish second-order guarantees for a class of federated learning algorithms.
\end{abstract}
\section{Introduction}
\label{sec:intro}
Federated learning pursues solutions to global optimization problems over distributed collections of agents by relying on the exchange of model updates in lieu of raw data. Federated architectures are frequently deployed in highly heterogeneous environments, where different agents have access to data of varying quality and varying computational resources. Performance guarantees for federated architectures are generally limited to convex loss functions, or to establishing limiting first-order stationarity on non-convex losses. First-order stationary points include minima, but can be saddle-points or local maxima as well. Saddle-points in particular have been identified as bottlenecks for optimization algorithms in many important applications, such as deep learning~\cite{Choromanska14, Kawaguchi16}. It is hence desirable to devise algorithms and performance analyses that ensure efficient escape from saddle-points despite high levels of asynchrony and heterogeneity. Recent works have identified gradient perturbations as playing a key role in guaranteeing efficient saddle-points escape in centralized and fully decentralized architectures~\cite{Ge15, HadiDaneshmand18, Jin19, Fang19, Vlaski19single, Vlaski19nonconvexP1, Vlaski19nonconvexP2}. Here, we establish analogous results in the federated learning framework, extending recent analysis from~\cite{Vlaski20Kailath} to allow for multiple local updates.

Specifically, we consider a collection of \( K \) agents, where each agent \( k \) is equipped with a risk loss function \( J_k(w) \), which is defined as the expectation of a loss \( Q(w; \x_k) \):
\begin{equation}\label{eq:def_q}
  J_k(w) \triangleq \E_{\x_k} Q(w; \x_k)
\end{equation}
Here, \( Q(w; \x_k) \) quantifies the fit of the model parametrization \( w \) to the random data \( \x_k \). Note that we allow for the data \( \x_k \) to vary with the agent index \( k \), resulting in different risk functions \( J_k(w) \) at different agents. It is common in multi-agent settings to pursue a model \( w^o \) that performs well on average by solving:
\begin{equation}\label{eq:aggregate_cost}
  w^o \triangleq \argmin_w \sum_{k=1}^K p_k J_k(w)
\end{equation}
where the \( \{ p_k \}_{k=1}^K \) denote non-negative weights, normalized to add up to one without loss of generality. It is common to let \( p_k = \frac{1}{K} \), hence giving equal weight to every agent \( k \). In settings where agents are heterogeneous, and exhibit varying amounts of data, or varying computational capabilities, heterogeneous weights \( p_k \) can result in improved performance, which we allow for generality. Perhaps the most straightforward approach to pursuing \( w^o \) is by means of gradient descent, applied directly to~\eqref{eq:aggregate_cost}, resulting in:
\begin{equation}\label{eq:gd}
  w_i = w_{i-1} - \mu \sum_{k=1}^{K} p_k \nabla J_k(w_{i-1}) = w_{i-1} - \mu \nabla J(w_{i-1})
\end{equation}
where we defined \( J(\cdot) \triangleq \sum_{k=1}^{K} p_k J_k(\cdot) \). This implementation has two important drawbacks, which render it impractical in a federated learning setting. First, it requires full agent participation at every iteration, by means of computation and communication of \( \nabla J_k(w_{i-1}) \) with a central aggregator. In federated learning applications, where agents may or may not be able to participate in the update at any given iteration, this can cause bottlenecks. Second, evaluation of the exact gradient \( \nabla J_k(w_{i-1}) \) may be infeasible or costly, since it depends on the full distribution of \( \x_k \) through its expectation in~\eqref{eq:def_q}.

\subsection{Related Works}
Distributed algorithms for solving aggregate optimization problems similar to~\eqref{eq:aggregate_cost} can be broadly classified into those that involve communication with a centralized parameter server~\cite{Bertsekas97parallel, Agarwal11, Stich19, Khaled19}, and those that operate in a fully decentralized manner through peer-to-peer interactions~\cite{Bertsekas97incremental, Nedic09, Sayed14, Sayed14proc, Duchi12}. Federated Averaging (FedAvg) was introduced in~\cite{mcmahan16}, and has sparked a number of studies and extensions, including FedDane~\cite{Li19}, FedProx~\cite{Tian20}, hierarchical FedAvg~\cite{Liu20}, and dynamic FedAvg~\cite{Rizk20}. While the pursuit of an optimal average model as in~\eqref{eq:aggregate_cost} is most common, multi-task variations have been introduced as well, both in a federated~\cite{Smith17} and decentralized settings~\cite{Nassif20}.

Most prior works on federated learning and the FedAvg algorithm focus on convex risk functions~\cite{Stich19, Khaled19, Rizk20}, or establish first-order stationarity in non-convex environments~\cite{Wang18, Zhou18, Hao19, Li19, Tian20, Liu20}. On the other hand, saddle points, which are first-order stationary, have been identified as bottlenecks in many learning applications, including deep learning~\cite{Choromanska14}. This contrast to the empirical success of deep learning has motivated a number of recent works to consider the ability of gradient descent algorithms to escape saddle-points and find ``good'' local minima, both in centralized~\cite{Gelfand91, Ge15, Du17, HadiDaneshmand18, Jin19, Fang19, Vlaski19single} and decentralized settings~\cite{Scutari18, Swenson19, Vlaski19nonconvexP1, Vlaski19nonconvexP2}. The broad take-away from these works is that perturbations, either to the initialization or gradient updates, play a key role in pushing iterates away from strict-saddle points and toward local minimizers. In this work, we extend these results to the federated learning setting, where agents may take an arbitrary number of local steps before communicating with the central parameter server.

\section{Algorithm Formulation}
\subsection{The Federated Averaging Scheme}
The need for full and exact agent participation in evaluating~\eqref{eq:gd} in a federated setting is addressed in the stochastic federated averaging (FedAvg) framework~\cite{mcmahan16}. To this end, the parameter server selects at iteration \( i \) a subset of \( L \) agents, collected in the set \( \mathcal{N}_i \). We introduce a random indicator variable \( \mathds{1}_{k, i} \), which indicates whether agent \( k \) participates at time \( i \), i.e., \( \mathds{1}_{k, i} = 1 \Longleftrightarrow k \in \mathcal{N}_i \), and \( 0 \) otherwise. We assume for simplicity that agents are sampled uniformly at random, resulting in:
\begin{equation}
  \mathrm{Pr}\left\{ \mathds{1}_{k, i} = 1 \right\} = \E\left\{ \mathds{1}_{k, i} \right\} = \frac{L}{K}
\end{equation}
Then, the parameter server provides participating agents with the current aggregate model \( \w_{i-1} \). They use the model to initialize their local iterate to \( \w_{k, 0} = \w_{i-1} \) and then perform \( E_k \) local stochastic update steps for \( e = 1, \ldots, E_k \):
\begin{equation}\label{eq:federated_local}
  \w_{k, e} = \w_{k, e-1} - \mu K \mathds{1}_{k, i}  \frac{p_k}{E_k} \widehat{\nabla J}_k^e(\w_{k, e-1})
\end{equation}
Here, \( \widehat{\nabla J}_k^e(\w_{k, e-1}) \) denotes a generic stochastic approximation of the gradient \( \nabla J_k(\w_{k, e-1}) \). Using realizations for the random variable \( \x_k \), it is common to construct \( \widehat{\nabla J}_k^e(\w_{k, e-1}) \triangleq \nabla Q(\w_{k, e-1}, \x_{k, e}) \), resulting in stochastic gradient descent --- we will discuss other constructions and their advantages in Section~\ref{sec:general_stochastic_approximation} below. The updated models are then fused by the central aggregator according to:
\begin{equation}\label{eq:federated_combination}
  \w_i = \frac{1}{L} \sum_{k=1}^K \mathds{1}_{k, i} \w_{k, E_k}
\end{equation}

\subsection{A General Stochastic Approximation Framework}\label{sec:general_stochastic_approximation}
We now present a number of choices for the stochastic gradient approximation \( \widehat{\nabla J}_k^e(\w_{k, e-1}) \) to illustrate the generality of~\eqref{eq:federated_local}.
\begin{example}[\textbf{Mini-Batch SGD}]\label{ex:sgd}
  Given a collection of \( B_k \) samples \( \left\{ \x_{k, e, b} \right\}_{b=1}^{B_k} \), constructing:
  \begin{equation}
    \widehat{\nabla J}_k^e(\w_{k, e-1}) = \frac{1}{B_k} \sum_{b=1}^{B_k} \nabla Q(\w_{k, e-1}, \x_{k, e, b})
  \end{equation}
  yields mini-batch stochastic gradient descent, or simply stochastic gradient descent when \( B_k = 1 \).\qed
\end{example}
\begin{example}[\textbf{Perturbed SGD}]\label{ex:psgd}
  It has been observed, both empirically and analytically, that adding additional perturbations to the stochastic gradient update can improve the performance of the gradient descent algorithm in non-convex settings~\cite{Jin19}. In the presence of privacy concerns, perturbations to update directions can also be added in order to ensure differential privacy~\cite{Dwork14}. This corresponds to constructing:
  \begin{equation}\label{eq:psgd}
    \widehat{\nabla J}_k^e(\w_{k, e-1}) = \nabla Q(\w_{k, e-1}, \x_{k, e}) + \boldsymbol{v}_{k, e}
  \end{equation}
  where \( \boldsymbol{v}_{k, e} \) denotes i.i.d. perturbation noise with zero mean, following for example a Gaussian or Laplacian distribution.\qed
\end{example}
\begin{example}[\textbf{Straggling Agents}]\label{ex:straggling}
  Consider a setting where agents may be unreliable, in the sense that, despite being chosen by the parameter server to participate at iteration \( i \), they may fail to return a locally updated model \( \w_{k, E_k} \) by the time the server needs to re-aggregate models in~\eqref{eq:federated_combination}. Such a setting can be modeled via:
  \begin{align}
    \widehat{\nabla J}_k^e(\w_{k, e-1}) = \begin{cases} \frac{1}{\delta_k} \nabla Q(\w_{k, e-1}, \x_{k, e})\ &\mathrm{with}\ \mathrm{prob.}\ \delta_k, \\ 0 \ &\mathrm{otherwise.} \end{cases}
  \end{align}
  Here, the scaling factor \( \frac{1}{\delta_k} \) has been added to ensure unbiased gradient approximations, by allowing agents who participate less frequently to take larger steps. Alternative stochastic models for asynchronous behavior are possible as well~\cite{Zhao15partI}.\qed
\end{example}
\noindent It can be readily verified, that all three constructions in Examples~\ref{ex:sgd}--\ref{ex:straggling} are unbiased approximations of the true gradient \( \nabla J_k(\w_{k, e-1}) \), i.e.:
\begin{equation}
  \E \left\{ \widehat{\nabla J}_k^e(\w_{k, e-1}) | \w_{k, e-1} \right\} = \nabla J_k(\w_{k, e-1})
\end{equation}
Nevertheless, the stochastic nature of the approximation induces a gradient noise into the evolution of the algorithm, which we denote by:
\begin{equation}\label{eq:local_gradient_noise}
  \s_{k, e}(\w_{k, e-1}) \triangleq \widehat{\nabla J}_k^e(\w_{k, e-1}) - \nabla J_k(\w_{k, e-1})
\end{equation}
We impose the following general conditions on the stochastic gradient noise process, and hence the construction of the stochastic gradient approximation itself.
\begin{assumption}[Gradient Noise Process]\label{as:approximation}
  The gradient noise process~\eqref{eq:local_gradient_noise} satisfies:
  \begin{align}
    \E\{\s_{k, e}(\w_{k, e-1})|\w_{k, e-1}\}&= 0 \label{eq:zero_mean} \\
    \E\{\|\s_{k, e}(\w_{k, e-1})\|^4|\w_{k, e-1}\}&\leq \beta_k^4 \|\nabla J_k(\w_{k, e-1})\|^4 + \sigma_k^4 \label{eq:bounded_fourth}
  \end{align}
  for some $\beta_k^4, \sigma_k^4\geq 0$. It is assumed that the gradient noise process is mutually independent over space and time, after conditioning on the current iterate:
  \begin{align}
    \mathds{E}\left\{ \s_{k_1, e_1}(\w) \s_{k_2, e_2}(\w)^{\mathsf{T}}| \w \right\} =&\: 0 \ \ \ \forall \ k_1 \neq k_2 \ \mathrm{or}\ e_1 \neq e_2
  \end{align}
  and the gradient noise covariance:
  \begin{equation}\label{eq:def_rs}
    R_{s, k}(\w_{k, e-1}) \triangleq \mathds{E}\left\{ \s_{k, e}(\w_{k, e-1}) \s_{k, e}(\w_{k, e-1})^{\mathsf{T}}| \w_{k, e-1} \right\}
  \end{equation}
  is smooth:
  \begin{equation}\label{eq:lipschitz_r}
      \| R_{s, k}(x) - R_{s, k}(y) \| \le \beta_R {\| x - y \|}^{\gamma}
    \end{equation}
    for some \( \beta_R \) and \( 0 < \gamma \le 4\), and there is a gradient noise component (in the aggregate) in every direction:
    \begin{equation}\label{eq:persistent_noise}
      R_{s, k}(x) \ge \sigma_{\ell}^2 I, \ \ \ \forall\ x
    \end{equation}\qed
\end{assumption}
\noindent Relation~\eqref{eq:zero_mean} ensures that the stochastic gradient approximation is unbiased, while~\eqref{eq:bounded_fourth} imposes a relative bound on the fourth-order moment~\cite{Sayed14}. In light of Jensen's inequality, it is stronger than imposing a bound on the gradient noise variance, but will allow us to more granularly study the impact of the gradient noise around saddle-points; on the other hand, it is weaker than the more common conditions of bounded noise with probability one, or a sub-Gaussian condition~\cite{Jin19, Fang19}. Relation~\eqref{eq:lipschitz_r} ensures that the distribution of the stochastic gradient noise process is locally smooth, allowing us to formulate an accurate short-term model around saddle-points~\cite{Vlaski19single}. It has previously been utilized to analyze in detail the steady-state behavior of stochastic gradient algorithms in convex environments~\cite{Sayed14}. The persistent noise condition~\eqref{eq:persistent_noise} will allow recursions to efficiently escape saddle-points by relying on the aggregate effect of the noise coupled with the local instability of saddle-points. It can be relaxed to only require a noise component to be present in the subspace of local descent directions~\cite{HadiDaneshmand18, Vlaski19single}. Since~\eqref{eq:persistent_noise} can always be ensured by adding a small amount of isotropic perturbations to the stochastic gradient update as in~\eqref{eq:psgd}, it will be sufficient, for simplicity, to impose~\eqref{eq:persistent_noise} in this work.

\section{Performance Analysis}
\subsection{A Perturbed Centralized Gradient Recursion}
By iterating~\eqref{eq:federated_local}, we find for the final local update \( \w_{k, E_k} \) sent back to the parameter server:
\begin{equation}
  \w_{k, E_k} = \w_{i-1} - \mu K \mathds{1}_{k, i} \frac{p_k}{E_k} \sum_{e=1}^{E_k} \widehat{\nabla J}_k^e(\w_{k, e-1})
\end{equation}
and after aggregation in~\eqref{eq:federated_combination}:
\begin{align}
  \w_{i} = \w_{i-1} - \mu \frac{K}{L} \sum_{k=1}^{K} \mathds{1}_{k, i} \frac{p_k}{E_k} \sum_{e=1}^{E_k} \widehat{\nabla J}_k^e(\w_{k, e-1})
\end{align}
We can reformulate this recursion to resemble the deterministic recursion~\eqref{eq:gd} as:
\begin{equation}\label{eq:federated_as_gd}
  \w_i = \w_{i-1} - \mu \sum_{k=1}^{K} p_k \nabla J_k(\w_{i-1}) - \mu \boldsymbol{s}_i - \mu \boldsymbol{d}_i
\end{equation}
where \( \boldsymbol{s}_i \) and \( \boldsymbol{d}_i \) are perturbation terms:
\begin{align}
  \s_i \triangleq&\: \frac{K}{L} \sum_{k=1}^K \mathds{1}_{k, i} \frac{p_k}{E_k} \sum_{e=1}^{E_k} \widehat{\nabla J}_k^e(\w_{i-1}) - \nabla J(\w_{i-1}) \label{eq:def_s}\\
  \boldsymbol{d}_{i} \triangleq&\: \frac{K}{L} \sum_{k=1}^K \mathds{1}_{k, i} \frac{p_k}{E_k} \sum_{e=1}^{E_k} \left( \widehat{\nabla J}_k^e(\w_{k, e-1}) - \widehat{\nabla J}_k^e(\w_{i-1}) \right) \label{eq:def_d}
\end{align}
Comparing~\eqref{eq:federated_as_gd} with~\eqref{eq:gd}, we observe that the FedAvg implementation can be viewed as a perturbed gradient descent recursion. Perturbations have recently been shown to be instrumental in allowing local descent algorithms to escape from saddle-points and converge to local minima of non-convex loss functions. However, those studies are generally limited to assuming unbiased perturbations. In contrast, employing~\eqref{eq:federated_local} with \( E_k > 1 \) results in biased gradient perturbations resulting from the term \( \boldsymbol{d}_i \), rendering current analyses inapplicable. In this work, we generalize recent results on the second-order guarantees of stochastic gradient algorithms~\cite{Vlaski19single} to allow for biased gradient perturbations, and recover second-order guarantees for the FedAvg algorithm for heterogeneous agents. We describe and discuss the dependence of these guarantees on the various parameters of the architecture, such as agent participation rate, levels heterogeneity, asynchrony, and computational capabilities.

We introduce the following smoothness conditions to ensure that the impact of the perturbations~\eqref{eq:def_s}--\eqref{eq:def_d} is limited.
\begin{assumption}[Smoothness]\label{as:smoothness}
  The local costs $J_k(w)$ are assumed to be smooth:
  \begin{align}
  \left \| \nabla J_k(x) - \nabla J_k(y) \right \| &\leq \delta \left \| x-y \right \|\\
  \left \| \nabla^2 J_k(x) - \nabla^2 J_k(y) \right \| &\leq \rho \left \| x-y \right \|
  \end{align}
  Heterogeneity between agents is quantified by their gradient disagreement:
  \begin{equation}\label{eq:disagreement}
    \| \nabla J_k(x) - \nabla J_{\ell}(x)\| \le G
  \end{equation}
  Furthermore, the costs themselves are assumed to be Lipschitz, implying uniformly bounded gradient:
  \begin{equation}\label{eq:bounded_gradient}
    \|\nabla J_k(x)\| \le U
  \end{equation}
  and the stochastic approximations of the gradient are Lipschitz in the mean-fourth sense:
  \begin{equation}\label{eq:ms_lipschitz}
    \E \left\{ {\left\| \widehat{\nabla J}_k^e(\x) - \widehat{\nabla J}_k^e(\y) \right \|}^4 | \x, \y \right\} \leq \hat{\delta}^4 \left \| \x-\y \right \|^4
  \end{equation}\qed
\end{assumption}

\subsection{Perturbation Bounds}
Under the conditions on the stochastic gradient construction in Assumption~\ref{as:approximation}, and the smoothness conditions in Assumption~\ref{as:smoothness} we can bound the perturbations~\eqref{eq:def_s}--\eqref{eq:def_d}.
\begin{lemma}[Perturbation Bounds]\label{LEM:PERTURBATION_BOUNDS}
  The perturbations to recursion~\eqref{eq:federated_as_gd} are bounded as:
  \begin{align}
    \E \left\{ \s_i | \w_{i-1} \right\} =&\: 0 \label{eq:aggregate_zero_mean} \\
    \E \left\{ \|\s_i\|^{4} | \w_{i-1} \right\} \le&\: \overline{\beta}^4 \|\nabla J(\w_{i-1})\|^4 + \overline{\sigma}^4 \label{eq:aggregate_bounded_fourth}\\
    \E \left\{ \|\boldsymbol{d}_i\|^{4} | \w_{i-1} \right\} \le&\: \mu^4 \sum_{k=1}^K p_k^5 \frac{K^6}{L^2} \hat{\delta}^4 8 \left(U^4+ \beta_k^4 U^4 + \sigma_k^4 \right) \label{eq:aggregate_bounded_incremental}
  \end{align}
  where we introduced the constants:
  \begin{align}
      \overline{\beta}^4 \triangleq&\: \sum_{k=1}^K p_k \overline{\beta}_k^4 \\
      \overline{\sigma}^4 \triangleq&\: \sum_{k=1}^K p_k \overline{\sigma}_k^4 \\
      \overline{\beta}_k^4 \triangleq&\: 192 \frac{K^3}{L^3} \frac{\beta_k^4}{E_k^2}  + 64 \frac{L}{K} \left( \frac{K-L}{L} \right)^4 + 64 \frac{K-L}{K} \\
      \overline{\sigma}_k^4 \triangleq&\: \left( 192 \frac{K^3}{L^3} \frac{\beta_k^4}{E_k^2} + 64 \frac{L}{K} \left( \frac{K-L}{L} \right)^4 + 64 \frac{K-L}{K} \right) G^4 \notag \\
      &\: + 24 \frac{K^3}{L^3} \frac{\sigma_k^4}{E_k^2}
    \end{align}
    The covariance of the aggregate gradient noise \( \s_i \) evaluates to:
    \begin{align}
      &\:\E \left\{ \s_i \s_i^{\T} | \w_{i-1} \right\} \notag \\
      =&\:\frac{K}{L}\sum_{k=1}^K \frac{p_k^2}{E_k} R_{s, k}(\w_{i-1}) + \frac{K}{L} \frac{K-L}{K-1} \sum_{k=1}^K {t}(\w_{i-1}){t}(\w_{i-1})^{\T} \notag \\
      \ge &\: \overline{\sigma}_{\ell}^2 I \triangleq \left( \sum_{k=1}^K \frac{p_k^2}{E_k} \right) \sigma_{\ell}^2 I
    \end{align}
    where \( {t}(\w_{i-1}) \) denotes the deviation:
    \begin{align}
      {t}(\w_{i-1}) \triangleq p_k \nabla J_k(\w_{i-1}) - \frac{1}{K} \nabla J(\w_{i-1})
    \end{align}
\end{lemma}
\begin{proof}
  Appendix~\ref{AP:PERTURBATION_BOUNDS}.
\end{proof}

\subsection{Second-Order Guarantees}
Examination of the bounds~\eqref{eq:aggregate_zero_mean}--\eqref{eq:aggregate_bounded_incremental} reveals that the aggregate zero-mean component \( \s_i \) arising from the use of stochastic gradient approximations continues to be bounded in a manner similar to the local approximations~\eqref{eq:bounded_fourth}, where the aggregate constant bounds are determined by the quality of local approximations \( \{\beta_k^4, \sigma_k^4\}_{k=1}^K \), the participation rate \( \frac{L}{K} \), the weights \( \{ p_k \}_{k=1}^K \), the level of heterogeneity \( G \), and number of local updates taken \( E_k \). The bias \( \boldsymbol{d}_i \) induced by employing multiple local updates, on the other hand, does not have zero-mean. The bound on its fourth-order moment~\eqref{eq:aggregate_bounded_incremental}, however, is proportional to \( \mu^4 \), causing its effect to be small for small step-sizes when compared to \( \s_{i} \), which is independent in \( \mu \). The fact that \( \boldsymbol{d}_i \) is biased renders traditional second-order analysis of stochastic gradient algorithms~\cite{Ge15, HadiDaneshmand18, Jin19, Vlaski20Kailath, Vlaski19single} inapplicable to this setting, while the fact that its fourth-order moment is small compared to \( \s_i \) makes it possible to extend the arguments of~\cite{Vlaski20Kailath, Vlaski19single}.
\begin{theorem}
  Suppose the aggregate loss \( J(w) \) is bounded from below by \( J(w) \ge J^o \). Then, with probability \( 1 - 2 \pi \):
  \begin{equation}\label{eq:first_guarantee}
    \| \nabla J(\w_{i^o}) \|^2 \le \mu \frac{\delta \overline{\sigma}^2}{1-2\mu \delta (1 + \overline{\beta}^2)} \left( 1 + \frac{1}{\pi} \right) + O(\mu^2)
  \end{equation}
  and \( \lambda_{\min}\left( \nabla^2 J(\w_{i^o}) \right) \ge -\tau \) in at most \( i^o \) iterations, where
  \begin{align}\label{eq:conv_guarantee}
    i^o \le \frac{2 \left( J(w_{0}) - J^o \right)}{\mu^2 \delta \overline{\sigma}^2} i^s
  \end{align}
  and \( i^s \) denotes the saddle-point escape time:
  \begin{align}
    i^{s} =  \frac{\log\left( 2 M  \frac{{\overline{\sigma}^2}}{{ \overline{\sigma}_{\ell}^2}} + 1 + O(\mu) \right)}{\mathrm{log}\left( 1 + 2 \mu \tau \right)}
  \end{align}
\end{theorem}
\begin{proof}
  The argument is an adjustment of~\cite{Vlaski19single} by bounding away the effect of \( \boldsymbol{d}_i \). Details omitted due to space limitations.
\end{proof}
\noindent This result ensures that, with probability \( 1 - 2 \pi \), the FedAvg algorithm will return a second-order stationary point with \( \| \nabla J(\w_{i^o}) \|^2 \le O(\mu) \) and \(  \lambda_{\min}\left( \nabla^2 J(\w_{i^o}) \right) \ge -\tau \) in at most \( i^o \) iterations, where \( i^o \) scales polynomially with all problem parameters. Every second-order stationary point, in light of \( \| \nabla J(\w_{i^o}) \|^2 \le O(\mu) \) is also first-order stationary, but the additional condition \( \lambda_{\min}\left( \nabla^2 J(\w_{i^o}) \right) \ge -\tau \) allows for the exclusion of strict saddle-points by choosing \( \tau \) sufficiently small.

\section{Numerical Results}
We illustrate the ability of the FedAvg algorithm to escape saddle-points for:
\begin{align}
  &\:Q(w_1, W_2;\boldsymbol{\gamma}, \boldsymbol{h}) \triangleq \log\left(1+e^{-\boldsymbol{\gamma}w_1^{\T}W_2 \boldsymbol{h}}\right) \\
  &\:J(w_1, W_2) \triangleq \E Q(w_1, W_2;\boldsymbol{\gamma}, \boldsymbol{h}) + \frac{\rho}{2} \|w_1\|^2 + \frac{\rho}{2} \|W_2\|^2
\end{align}
This loss arises when training a neural network with a single, linear hidden layer to predict the class label \( \boldsymbol{\gamma} \) from \( \boldsymbol{h} \) using cross-entropy, and exhibits a strict saddle-point at \( w_1 = W_2 = 0 \), making it suitable as a simplified benchmark --- see~\cite{Vlaski19nonconvexP1} for a discussion and motivation. For a total of \( K = 100 \) agents, we vary the rate of participation from \( L =1 \) to \( L =100 \). Agents are chosen uniformly, and participating agents perform \( E = 10  \) local updates constructed as a combination of Examples~\ref{ex:psgd} and~\ref{ex:straggling}, namely:
\begin{align}
  \widehat{\nabla J}_k^e(\w_{k, e-1}) = \frac{1}{\delta_k} \left( \nabla Q(\w_{k, e-1}, \x_{k, e}) - \rho \w_{k, e-1} \right)
\end{align}
with probability \( p_k = 0.5 \), and \( \widehat{\nabla J}_k^e(\w_{k, e-1}) = 0  \) otherwise. Evolution of iterates and the gradient norm are shown in Figures~\ref{fig:evolution} and~\ref{fig:gradient_norm} respectively.
\begin{figure}[!t]
	\centering
	\includegraphics[width=\columnwidth]{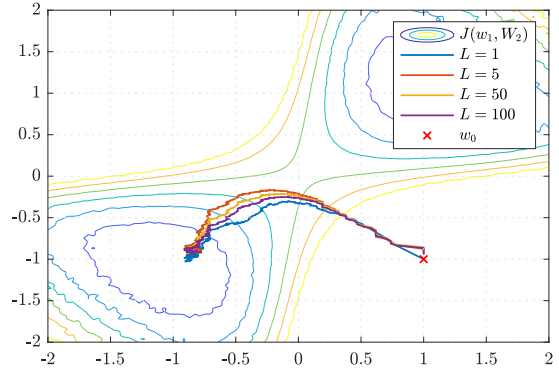}
	\caption{Evolution of the aggregate model for various choices of the participation rate \( \frac{L}{K}\). All implementations escape the saddle-point and find a local minimum.}\label{fig:evolution}
\end{figure}
\begin{figure}[!t]
	\centering
	\includegraphics[width=\columnwidth]{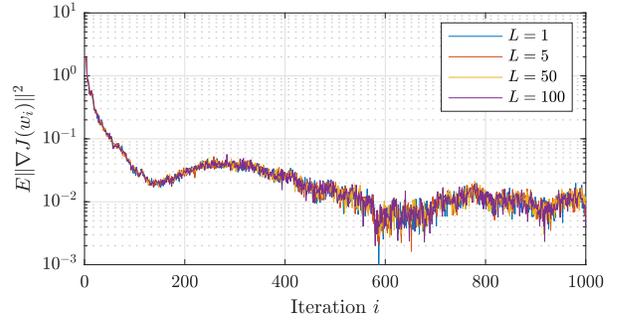}
	\caption{Evolution of the gradient norm for varying participation rates.}\label{fig:gradient_norm}
\end{figure}

\section{Conclusion}
In this work, we considered a highly heterogeneous and asynchronous variant of the Federated Averaging (FedAvg) algorithm, where agents may be using varying, potentially unreliable, stochastic gradient approximations with varying quality, and take a different number \( E_k \) of local update steps, and established convergence to second-order stationary points. Despite high levels of heterogeneity and asynchrony, the algorithm continues to escape saddle-points and return second-order stationary points in polynomial time, shedding light on the success of deep learning, which is frequently employed in federated learning settings.

\bibliographystyle{IEEEbib}
{\bibliography{main}}

\appendix\allowdisplaybreaks[4]
\section{Proof of Lemma~\ref{LEM:PERTURBATION_BOUNDS}}\label{AP:PERTURBATION_BOUNDS}
We begin by establishing that \( \s_{i} \) has conditional zero-mean:
\begin{align}
  &\: \E \left\{ \s_i | \w_{i-1} \right\} \notag \\
  \stackrel{\eqref{eq:def_s}}{=}&\: \frac{K}{L} \sum_{k=1}^K \frac{p_k}{E_k} \sum_{e=1}^{E_k} \E \left\{ \mathds{1}_{k, i} \widehat{\nabla J}_k^e(\w_{i-1}) | \w_{i-1} \right\} - \nabla J(\w_{i-1}) \notag \\
  \stackrel{(a)}{=}&\: \frac{K}{L} \sum_{k=1}^K \frac{p_k}{E_k} \sum_{e=1}^{E_k} \E \left\{ \mathds{1}_{k, i} \right\} \E \left\{ \widehat{\nabla J}_k^e(\w_{i-1}) | \w_{i-1} \right\} - \nabla J(\w_{i-1}) \notag \\
  \stackrel{(b)}{=}&\: \frac{K}{L} \sum_{k=1}^K \frac{p_k}{E_k} \sum_{e=1}^{E_k} \frac{L}{K} {\nabla J}_k(\w_{i-1}) - \nabla J(\w_{i-1}) \notag \\
  =&\: \sum_{k=1}^K p_k {\nabla J}_k(\w_{i-1}) - \nabla J(\w_{i-1}) = 0 \label{eq:zero_mean_appendix}
\end{align}
where \( (a) \) follows because participation \( \mathds{1}_{k, i} \) is independent of \( \w_{i-1} \) and the data available at time \( i \), and hence \( \widehat{\nabla J}_k^e(\w_{i-1}) \). Step \( (b) \) follows from \( \E \left\{ \mathds{1}_{k, i} \right\} = \frac{L}{K} \) and~\eqref{eq:zero_mean}.
We now proceed to evaluate the aggregate gradient noise covariance. For brevity, we define:
\begin{align}\label{eq:def_g}
  \boldsymbol{g}_{k, i} \triangleq \frac{1}{E_k} \sum_{e=1}^{E_k} \widehat{\nabla J}_k^e(\w_{i-1})
\end{align}
Then:
\begin{align}
  \s_{i} \triangleq \frac{K}{L} \sum_{k=1}^K \mathds{1}_{k, i} p_k \boldsymbol{g}_{k, i} - \nabla J(\w_{i-1})
\end{align}
For the aggregate gradient noise covariance, we have:
\begin{align}
  &\: \E \left\{ \s_i \s_i^{\T} | \w_{i-1} \right\}\notag \\
  =&\: \E \Bigg\{ \left( \frac{K}{L} \sum_{k=1}^K \mathds{1}_{k, i} p_k \boldsymbol{g}_{k, i} - \nabla J(\w_{i-1}) \right) \notag \\
  &\: \ \ \ \ \ \ \ \times {\left( \frac{K}{L} \sum_{k=1}^K \mathds{1}_{k, i} p_k \boldsymbol{g}_{k, i} - \nabla J(\w_{i-1}) \right)}^{\T} | \w_{i-1} \Bigg\}\notag \\
  \stackrel{(a)}{=}&\: \E \Bigg\{ \left( \frac{K}{L} \sum_{k=1}^K \mathds{1}_{k, i} p_k \boldsymbol{g}_{k, i} \right)  {\left( \frac{K}{L} \sum_{k=1}^K \mathds{1}_{k, i} p_k \boldsymbol{g}_{k, i} \right)}^{\T} | \w_{i-1} \Bigg\}\notag \\
  &\: - \nabla J(\w_{i-1}) {\nabla J(\w_{i-1})}^{\T} \notag \\
  {=}&\: \E \Bigg\{ \left( \frac{K^2}{L^2} \sum_{k=1}^K \sum_{\ell=1}^K \mathds{1}_{k, i} \mathds{1}_{\ell, i} p_k p_{\ell} \boldsymbol{g}_{k, i} \boldsymbol{g}_{\ell, i}^{\T} \right) | \w_{i-1} \Bigg\} \notag \\
  &\: - \nabla J(\w_{i-1}) {\nabla J(\w_{i-1})}^{\T}\label{eq:intermediate_cross_terms}
\end{align}
where \( (a) \) follows after multiplying and simplifying cross-terms by noting that:
\begin{equation}
  \E \left\{\frac{K}{L} \sum_{k=1}^K \mathds{1}_{k, i} p_k \boldsymbol{g}_{k, i} | \w_{i-1} \right\} \stackrel{\eqref{eq:zero_mean_appendix}}{=} \nabla J(\w_{i-1})
\end{equation}
The challenge in evaluating~\eqref{eq:intermediate_cross_terms} lies in the fact that, while the approximations \( \boldsymbol{g}_{k, i} \) and \( \boldsymbol{g}_{\ell, i} \) are mutually independent by Assumption~\ref{as:approximation}, the same does not hold for the participation indicators \( \mathds{1}_{k, i} \) and \( \mathds{1}_{\ell, i} \), since agents are sampled without replacement. We can nevertheless evaluate:
\begin{align}
  &\: \E \left\{ \left( \frac{K^2}{L^2} \sum_{k=1}^K \sum_{\ell=1}^K \mathds{1}_{k, i} \mathds{1}_{\ell, i} p_k p_{\ell} \boldsymbol{g}_{k, i} \boldsymbol{g}_{\ell, i}^{\T} \right) | \w_{i-1} \right\} \notag \\
  \stackrel{(a)}{=}&\: \frac{K^2}{L^2} \sum_{k=1}^K \sum_{\ell=1}^K \E \left\{  \mathds{1}_{k, i} \mathds{1}_{\ell, i} p_k p_{\ell} \boldsymbol{g}_{k, i} \boldsymbol{g}_{\ell, i}^{\T} | \w_{i-1}, \mathds{1}_{k, i} = 1\right\} \notag \\
  &\: \times \mathrm{Pr} \left\{ \mathds{1}_{k, i} = 1 \right\} \notag \\
  =&\: \frac{L}{K} \frac{K^2}{L^2} \sum_{k=1}^K \sum_{\ell=1}^K \E \left\{ \mathds{1}_{\ell, i} p_k p_{\ell} \boldsymbol{g}_{k, i} \boldsymbol{g}_{\ell, i}^{\T} | \w_{i-1}, \mathds{1}_{k, i} = 1\right\} \notag \\
  \stackrel{(b)}{=}&\: \frac{K}{L} \sum_{k=1}^K \sum_{\ell=1}^K \E \left\{  \mathds{1}_{\ell, i} | \mathds{1}_{k, i} = 1 \right\} p_k p_{\ell}\E \left\{ \boldsymbol{g}_{k, i} \boldsymbol{g}_{\ell, i}^{\T} | \w_{i-1}\right\} \notag \\
  \stackrel{(c)}{=}&\: \frac{K}{L} \sum_{k=1}^K \E \left\{  \mathds{1}_{k, i} | \mathds{1}_{k, i} = 1 \right\} p_k^2 \E \left\{ \boldsymbol{g}_{k, i} \boldsymbol{g}_{k, i}^{\T} | \w_{i-1}\right\} \notag \\
  &\: + \frac{K}{L}  \sum_{k=1}^K \sum_{\ell \neq k} \E \left\{  \mathds{1}_{\ell, i} | \mathds{1}_{k, i} = 1 \right\} p_k p_{\ell} \E \left\{ \boldsymbol{g}_{k, i} \boldsymbol{g}_{\ell, i}^{\T} | \w_{i-1}\right\} \notag \\
  \stackrel{(d)}{=}&\: \frac{K}{L}  \sum_{k=1}^K p_k^2 \E \left\{  \boldsymbol{g}_{k, i} \boldsymbol{g}_{k, i}^{\T}  | \w_{i-1}\right\} \notag \\
  &\: + \frac{K}{L}\frac{L-1}{K-1}   \sum_{k=1}^K \sum_{\ell \neq k} p_k p_{\ell} \E \left\{\boldsymbol{g}_{k, i} \boldsymbol{g}_{\ell, i}^{\T}  | \w_{i-1}\right\}
\end{align}
where \( (a) \) follows from Bayes' theorem and \( (b) \) is due to the fact that \( \mathds{1}_{k, i} \) and \( \w_{i-1} \) are independent, \( (c) \) separates cross-terms and \( (d) \) results from \( \E \left\{  \mathds{1}_{\ell, i} | \mathds{1}_{k, i} = 1 \right\} = \frac{L-1}{K-1} \). We return to~\eqref{eq:intermediate_cross_terms}:
\begin{align}
  &\: \E \left\{ \s_i \s_i^{\T} | \w_{i-1} \right\}\notag \\
  =&\: \frac{K}{L}  \sum_{k=1}^K p_k^2 \E \left\{  \boldsymbol{g}_{k, i} \boldsymbol{g}_{k, i}^{\T}  | \w_{i-1}\right\} \notag \\
  &\: - \left( \sum_{k=1}^K p_k \nabla J_k(\w_{i-1}) \right) {\left( \sum_{\ell=1}^K p_{\ell} \nabla J_{\ell}(\w_{i-1}) \right)}^{\T} \notag \\
  &\: + \frac{K}{L}\frac{L-1}{K-1}   \sum_{k=1}^K \sum_{\ell \neq k} p_k p_{\ell} \E \left\{\boldsymbol{g}_{k, i} \boldsymbol{g}_{\ell, i}^{\T}  | \w_{i-1}\right\} \notag \\
  \stackrel{(a)}{=}&\: \frac{K}{L}  \sum_{k=1}^K p_k^2 \E \left\{  \boldsymbol{g}_{k, i} \boldsymbol{g}_{k, i}^{\T}  | \w_{i-1}\right\}  - \sum_{k=1}^K p_k^2 \nabla J_k(\w_{i-1}) \nabla J_k(\w_{i-1})^{\T} \notag \\
  &\: - \sum_{k=1}^K \sum_{\ell \neq k} p_k p_{\ell} \nabla J_k(\w_{i-1})  \nabla J_{\ell}(\w_{i-1})^{\T} \notag \\
  &\: + \frac{K}{L}\frac{L-1}{K-1}   \sum_{k=1}^K \sum_{\ell \neq k} p_k p_{\ell} \E \left\{\boldsymbol{g}_{k, i} \boldsymbol{g}_{\ell, i}^{\T}  | \w_{i-1}\right\} \notag \\
  \stackrel{(b)}{=}&\: \frac{K}{L} \sum_{k=1}^K p_k^2 \E \Big\{ \left( \boldsymbol{g}_{k, i} - \nabla J_k(\w_{i-1}) \right) \notag \\
  &\: \ \ \ \ \ \ \ \ \ \ \ \ \ \ \ \ \ \ \ \ \ \ \ \ \ \ \times \left( \boldsymbol{g}_{k, i} - \nabla J_k(\w_{i-1}) \right)^{\T}  | \w_{i-1}\Big\} \notag \\
  &\: + \frac{K-L}{L}  \sum_{k=1}^K p_k^2 \nabla J_k(\w_{i-1}) \nabla J_k(\w_{i-1})^{\T} \notag \\
  &\: + \left( \frac{K}{L}\frac{L-1}{K-1} - 1 \right) \sum_{k=1}^K \sum_{\ell \neq k} p_k p_{\ell} \nabla J_k(\w_{i-1})  \nabla J_{\ell}(\w_{i-1})^{\T} \notag \\
  &\: - \frac{K}{L}\frac{L-1}{K-1} \sum_{k=1}^K \sum_{\ell \neq k} p_k p_{\ell} \nabla J_k(\w_{i-1})  \nabla J_{\ell}(\w_{i-1})^{\T} \notag \\
  &\: + \frac{K}{L}\frac{L-1}{K-1}   \sum_{k=1}^K \sum_{\ell \neq k} p_k p_{\ell} \E \left\{\boldsymbol{g}_{k, i} \boldsymbol{g}_{\ell, i}^{\T}  | \w_{i-1}\right\} \notag \\
  \stackrel{(c)}{=}&\: \frac{K}{L}\sum_{k=1}^K \frac{p_k^2}{E_k} R_{s, k}(\w_{i-1}) \notag \\
  &\:+ \frac{K-L}{L}  \sum_{k=1}^K p_k^2 \nabla J_k(\w_{i-1}) \nabla J_k(\w_{i-1})^{\T} \notag \\
  &\: - \frac{K-L}{L (K-1)} \sum_{k=1}^K \sum_{\ell \neq k} p_k p_{\ell} \nabla J_k(\w_{i-1})  \nabla J_{\ell}(\w_{i-1})^{\T} \notag \\
  \stackrel{(d)}{=}&\: \frac{K}{L}\sum_{k=1}^K \frac{p_k^2}{E_k} R_{s, k}(\w_{i-1}) \notag \\
  &\: + \frac{K-L}{L}  \sum_{k=1}^K p_k^2 \nabla J_k(\w_{i-1}) \nabla J_k(\w_{i-1})^{\T} \notag \\
  &\: + \frac{K-L}{L (K-1)} \sum_{k=1}^K p_k^2 \nabla J_k(\w_{i-1}) \nabla J_k(\w_{i-1})^{\T} \notag \\
  &\: - \frac{K-L}{L (K-1)} \nabla J(\w_{i-1}) \nabla J(\w_{i-1})^{\T} \notag \\
  =&\: \frac{K}{L}\sum_{k=1}^K \frac{p_k^2}{E_k} R_{s, k}(\w_{i-1}) \notag \\
  &\: + \frac{K}{L} \frac{K-L}{K-1} \sum_{k=1}^K p_k^2 \nabla J_k(\w_{i-1}) \nabla J_k(\w_{i-1})^{\T} \notag \\
  &\: - \frac{K-L}{L (K-1)} \nabla J(\w_{i-1}) \nabla J(\w_{i-1})^{\T} \notag \\
  \stackrel{(e)}{=}&\: \frac{K}{L}\sum_{k=1}^K \frac{p_k^2}{E_k} R_{s, k}(\w_{i-1})   \notag \\
  &\: + \frac{K}{L} \frac{K-L}{K-1} \sum_{k=1}^K \Bigg( \left( p_k \nabla J_k(\w_{i-1}) - \frac{1}{K} \nabla J(\w_{i-1}) \right) \notag \\
  &\: \ \ \ \ \ \ \ \ \ \ \ \ \ \ \ \ \ \ \ \ \ \ \ \times \left( p_k \nabla J_k(\w_{i-1}) - \frac{1}{K} \nabla J(\w_{i-1}) \right)^{\T} \Bigg)
\end{align}
where \( (a) \) multiplies \( \left( \sum_{k=1}^K p_k \nabla J_k(\w_{i-1}) \right) {\left( \sum_{\ell=1}^K p_{\ell} \nabla J_{\ell}(\w_{i-1}) \right)}^{\T} \) and separates cross-terms, \( (b) \) combines terms using the fact that \( \E \left\{ \boldsymbol{g}_{k, i} | \w_{i-1} \right\} = \nabla J_k(\w_{i-1}) \), \( (c) \) follows from~\eqref{eq:def_g},~\eqref{eq:def_rs} and the fact the \( \boldsymbol{g}_{k, i} \) are mutually independent. Step \( (d) \) completes the square to obtain \( \nabla J(\w_{i-1}) \nabla J(\w_{i-1})^{\T} \) and \( (e) \) can be verified by multiplying out the result. For the fourth-order moment, we have following the argument in~\cite[Example 7]{Vlaski20Kailath}:
\begin{align}
  &\: \E \left\{{\| \s_i\|}^4 | \w_{i-1} \right\} \notag \\
  \stackrel{\eqref{eq:def_s}}{=}&\: \E \left\{\left\| \frac{K}{L} \sum_{k=1}^K \frac{p_k}{E_k} \sum_{e=1}^{E_k} \mathds{1}_{k, i} \widehat{\nabla J}_k^e(\w_{i-1}) - \nabla J(\w_{i-1}) \right\|^4 | \w_{i-1} \right\} \notag \\
  {=}&\: \E \Bigg\{ \Bigg\| \frac{K}{L} \sum_{k=1}^K \frac{p_k}{E_k} \sum_{e=1}^{E_k} \mathds{1}_{k, i} \widehat{\nabla J}_k^e(\w_{i-1}) \notag \\
  &\: \ \ \ \ \ \ \ \ \ \ - \sum_{k=1}^K p_k \nabla J_k(\w_{i-1}) \Bigg\|^4 | \w_{i-1} \Bigg\} \notag \\
  \stackrel{(a)}{\le}&\:\sum_{k=1}^K p_k  \E \Bigg\{ \Bigg\| \frac{K}{L} \frac{1}{E_k} \sum_{e=1}^{E_k} \mathds{1}_{k, i} \widehat{\nabla J}_k^e(\w_{i-1}) \notag \\
  &\: \ \ \ \ \ \ \ \ \ \ \ \ \ \ \ \ \ \ \ \ \ - \nabla J_k(\w_{i-1}) \Bigg\|^4 | \w_{i-1} \Bigg\}
\end{align}
where \( (a) \) follows by Jensen's inequality. We proceed with the individual terms of the sum:
\begin{align}
  &\:\E \left\{ \left\| \frac{K}{L} \frac{1}{E_k} \sum_{e=1}^{E_k} \mathds{1}_{k, i} \widehat{\nabla J}_k^e(\w_{i-1}) - \nabla J_k(\w_{i-1}) \right\|^4 | \w_{i-1} \right\} \notag \\
  =&\:\E \Bigg\{ \Bigg\| \frac{K}{L} \frac{1}{E_k} \sum_{e=1}^{E_k} \mathds{1}_{k, i} \left( \widehat{\nabla J}_k^e(\w_{i-1}) - \nabla J_k(\w_{i-1})  \right) \notag \\
  &\: \ \ \ \ \ \ \ \ \ \ + \left( \frac{K}{L} \mathds{1}_{k, i} - 1 \right) \nabla J_k(\w_{i-1}) \Bigg\|^4 | \w_{i-1} \Bigg\} \notag \\
  \stackrel{\eqref{eq:local_gradient_noise}}{=}&\:\E \Bigg\{ \Bigg\| \frac{K}{L} \frac{1}{E_k} \sum_{e=1}^{E_k} \mathds{1}_{k, i} \s_{k, e}(\w_{i-1}) \notag \\
  &\: \ \ \ \ \ \ \ \ \ \ + \left( \frac{K}{L} \mathds{1}_{k, i} - 1 \right) \nabla J_k(\w_{i-1}) \Bigg\|^4 | \w_{i-1} \Bigg\} \notag \\
  \stackrel{(a)}{\le}&\: 8 \E \left\{ \left\| \frac{K}{L} \frac{1}{E_k} \sum_{e=1}^{E_k} \mathds{1}_{k, i} \s_{k, e}(\w_{i-1}) \right\|^4 | \w_{i-1} \right\} \notag \\
  &\:+ 8 \E \left\{ \left\| \left( \frac{K}{L} \mathds{1}_{k, i} - 1 \right) \nabla J_k(\w_{i-1}) \right\|^4 | \w_{i-1} \right\} \notag \\
  \stackrel{(b)}{=}&\: 8 \frac{K^4}{L^4} \E \left\{ \left\| \frac{1}{E_k} \sum_{e=1}^{E_k} \mathds{1}_{k, i} \s_{k, e}(\w_{i-1}) \right\|^4 | \w_{i-1} \right\} \notag \\
  &\:+ 8 \E \left\{ \left( \frac{K}{L} \mathds{1}_{k, i} - 1 \right)^4 \right\} \left\|  \nabla J_k(\w_{i-1}) \right\|^4 \notag \\
  \stackrel{(c)}{=}&\: 8 \frac{K^4}{L^4} \mathrm{Pr}\left\{ \mathds{1}_{k, i} = 1 \right\} \E \left\{ \left\| \frac{1}{E_k} \sum_{e=1}^{E_k} \s_{k, e}(\w_{i-1}) \right\|^4 | \w_{i-1} \right\} \notag \\
  &\:+ 8 \mathrm{Pr}\left\{ \mathds{1}_{k, i} = 1 \right\} \E \left\{\left( \frac{K}{L} - 1 \right)^4 \right\} \left\|  \nabla J_k(\w_{i-1}) \right\|^4 \notag \\
  &\:+ 8 \mathrm{Pr}\left\{ \mathds{1}_{k, i} = 0 \right\} \E \left\{\left( - 1 \right)^4 \right\} \left\|  \nabla J_k(\w_{i-1}) \right\|^4 \notag \\
  \stackrel{(d)}{=}&\: 8 \frac{K^4}{L^4} \frac{L}{K} \E \left\{ \left\| \frac{1}{E_k} \sum_{e=1}^{E_k} \s_{k, e}(\w_{i-1}) \right\|^4 | \w_{i-1} \right\} \notag \\
  &\:+ 8 \left( \frac{L}{K} \left( \frac{K-L}{L} \right)^4 + \frac{K-L}{K} \right) \left\|  \nabla J_k(\w_{i-1}) \right\|^4 \notag \\
  \stackrel{(e)}{\le}&\: 24 \frac{K^3}{L^3} \frac{1}{E_k^2} \left( \beta_k^4 \left\|  \nabla J_k(\w_{i-1}) \right\|^4 + \sigma_k^4 \right) \notag \\
  &\:+ 8 \left( \frac{L}{K} \left( \frac{K-L}{L} \right)^4 + \frac{K-L}{K} \right) \left\|  \nabla J_k(\w_{i-1}) \right\|^4 \notag \\
  =&\: \left( 24 \frac{K^3}{L^3} \frac{\beta_k^4}{E_k^2}  + 8 \frac{L}{K} \left( \frac{K-L}{L} \right)^4 + 8 \frac{K-L}{K} \right) \left\|  \nabla J_k(\w_{i-1}) \right\|^4  \notag \\
  &\: + 24 \frac{K^3}{L^3} \frac{\sigma_k^4}{E_k^2} \notag \\
  =&\: \left( 24 \frac{K^3}{L^3} \frac{\beta_k^4}{E_k^2}  + 8 \frac{L}{K} \left( \frac{K-L}{L} \right)^4 + 8 \frac{K-L}{K} \right) \notag \\
  &\: \times \left\| \nabla J(\w_{i-1}) + \nabla J_k(\w_{i-1}) - \nabla J(\w_{i-1}) \right\|^4  \notag \\
  &\: + 24 \frac{K^3}{L^3} \frac{\sigma_k^4}{E_k^2} \notag \\
  \stackrel{(f)}{\le}&\: \left( 24 \frac{K^3}{L^3} \frac{\beta_k^4}{E_k^2}  + 8 \frac{L}{K} \left( \frac{K-L}{L} \right)^4 + 8 \frac{K-L}{K} \right) \notag \\
  &\: \times \left( 8 \left\| \nabla J(\w_{i-1}) \right\|^4 + 8 \left\| \nabla J_k(\w_{i-1}) - \nabla J(\w_{i-1}) \right\|^4 \right)  \notag \\
  &\: + 24 \frac{K^3}{L^3} \frac{\sigma_k^4}{E_k^2} \notag \\
  \stackrel{\eqref{eq:disagreement}}{\le}&\: \left( 24 \frac{K^3}{L^3} \frac{\beta_k^4}{E_k^2}  + 8 \frac{L}{K} \left( \frac{K-L}{L} \right)^4 + 8 \frac{K-L}{K} \right) \notag \\
  &\: \times \left( 8 \left\| \nabla J(\w_{i-1}) \right\|^4 + 8 G^4 \right) + 24 \frac{K^3}{L^3} \frac{\sigma_k^4}{E_k^2} \notag \\
  =&\: \left( 192 \frac{K^3}{L^3} \frac{\beta_k^4}{E_k^2}  + 64 \frac{L}{K} \left( \frac{K-L}{L} \right)^4 + 64 \frac{K-L}{K} \right) \left\| \nabla J(\w_{i-1}) \right\|^4\notag \\
  &\:+ \left( 192 \frac{K^3}{L^3} \frac{\beta_k^4}{E_k^2} + 64 \frac{L}{K} \left( \frac{K-L}{L} \right)^4 + 64 \frac{K-L}{K} \right) G^4 \notag \\
  &\: + 24 \frac{K^3}{L^3} \frac{\sigma_k^4}{E_k^2}
\end{align}
where \( (a) \) and \( (f) \) follow from Jensen's inequality, \( (b) \) uses the fact that \( \mathds{1}_{k, i} \) is independent of \( \widehat{\nabla J}_k^e(\w_{i-1}) \), \( (c) \) applies Bayes' theorem and \( (d) \) uses \( \mathrm{Pr}\left\{ \mathds{1}_{k, i} = 1 \right\} = \E \left\{ \mathds{1}_{k, i} \right\} = \frac{L}{K} \). Step \( (e) \) follows from~\eqref{eq:bounded_fourth} and:
\begin{align}
  &\:\E \left\{ \left\| \frac{1}{E_k} \sum_{e=1}^{E_k} \s_{k, e}(\w_{i-1}) \right\|^4 | \w_{i-1} \right\} \notag \\
  \le&\: \frac{3 - \frac{2}{E_k}}{E_k^2} \left( \beta_k^4 \left\|  \nabla J_k(\w_{i-1}) \right\|^4 + \sigma_k^4 \right),
\end{align}
which can be verified by induction over \( E_k \)~\cite{Vlaski20Kailath}. Next, we bound the fourth-moment of the term \( \boldsymbol{d}_i \), arising from the fact that agents take \( E_k > 1 \) local gradient steps before returning the updated estimate to the parameter server. We introduce \( \boldsymbol{d}_{k, e-1} \triangleq \widehat{\nabla J}_k^e(\w_{k, e-1}) - \widehat{\nabla J}_k^e(\w_{i-1}) \) for brevity. Then, we have:
\begin{align}
  &\: \E \left\{{\| \boldsymbol{d}_i\|}^4 | \w_{i-1} \right\} \notag \\
  \stackrel{\eqref{eq:def_d}}{=}&\: \E \left\{{\left\| \frac{K}{L} \sum_{k=1}^K \mathds{1}_{k, i} \frac{p_k}{E_k} \sum_{e=1}^{E_k} \boldsymbol{d}_{k, e-1} \right\|}^4 | \w_{i-1} \right\} \notag \\
  \stackrel{(a)}{\le}&\: \sum_{k=1}^K p_k \E \left\{{\left\| \frac{K}{L}  \frac{\mathds{1}_{k, i}}{E_k} \sum_{e=1}^{E_k} \boldsymbol{d}_{k, e-1} \right\|}^4 | \w_{i-1} \right\} \notag \\
  {=}&\: \sum_{k=1}^K p_k \frac{K^4}{L^4} \E \left\{{\left\| \frac{\mathds{1}_{k, i}}{E_k} \sum_{e=1}^{E_k} \boldsymbol{d}_{k, e-1} \right\|}^4 | \w_{i-1} \right\} \notag \\
  \stackrel{(b)}{=}&\: \sum_{k=1}^K p_k \frac{K^4}{L^4} \mathrm{Pr}\left\{ \mathds{1}_{k, i} = 1 \right\}\E \left\{{\left\| \frac{1}{E_k} \sum_{e=1}^{E_k} \boldsymbol{d}_{k, e-1} \right\|}^4 | \w_{i-1} \right\} \notag \\
  \stackrel{(c)}{=}&\: \sum_{k=1}^K p_k \frac{K^3}{L^3}\E \left\{{\left\| \frac{1}{E_k} \sum_{e=1}^{E_k} \boldsymbol{d}_{k, e-1} \right\|}^4 | \w_{i-1} \right\} \notag \\
  \stackrel{(d)}{\le}&\: \sum_{k=1}^K p_k \frac{K^3}{L^3} \frac{1}{E_k} \sum_{e=1}^{E_k} \E \left\{{\left\| \boldsymbol{d}_{k, e-1} \right\|}^4 | \w_{i-1} \right\} \notag \\
  \stackrel{\eqref{eq:ms_lipschitz}}{\le}&\: \sum_{k=1}^K p_k \frac{K^3}{L^3} \frac{\hat{\delta}^4}{E_k} \sum_{e=1}^{E_k} \E \left\{{\left\| \w_{k, e-1} - \w_{i-1}\right\|}^4 | \w_{i-1} \right\}\label{eq:intermediate}
\end{align}
where \( (a) \) and \( (d) \) follow from Jensen's inequality, \( (b) \) applies a Bayes' decomposition and \( (c) \) follows form \( \mathrm{Pr} \left\{\mathds{1}_{k, i} =1  \right\} = \frac{L}{K} \). We now bound the deviation of estimates over one epoch. For \( e = 1\), we have \( \w_{k, e-1} = \w_{k, 0} = \w_{i-1} \) and hence \( \E \left\{{\left\| \w_{k, e-1} - \w_{i-1}\right\|}^4 | \w_{i-1} \right\} = 0 \). For \( e \ge 2 \), iterating~\eqref{eq:federated_local}, we find:
\begin{align}
  &\:\E \left\{{\left\| \w_{k, e-1} - \w_{i-1} \right\|}^4 | \w_{i-1} \right\} \notag\\
  =&\:\E \left\{{\left\| \mu K \mathds{1}_{k, i} \frac{p_k}{E_k} \sum_{j=1}^{e-1}  \widehat{\nabla J}_k^e(\w_{k, j-1}) \right\|}^4 | \w_{i-1} \right\} \notag \\
  =&\: \mu^4 p_k^4 K^4 \E \left\{{\left\| \mathds{1}_{k, i} \frac{1}{E_k} \sum_{j=1}^{e-1}  \widehat{\nabla J}_k^e(\w_{k, j-1}) \right\|}^4 | \w_{i-1} \right\} \notag \\
  \stackrel{(a)}{=}&\: \mu^4 p_k^4 K^4 \mathrm{Pr}\left\{\mathds{1}_{k, i} =1 \right\} \E \left\{{\left\| \frac{1}{E_k} \sum_{j=1}^{e-1}  \widehat{\nabla J}_k^e(\w_{k, j-1}) \right\|}^4 | \w_{i-1} \right\} \notag \\
  \stackrel{(b)}{=}&\: \mu^4 p_k^4 K^4 \frac{L}{K} \E \left\{{\left\| \frac{1}{E_k} \sum_{j=1}^{e-1}  \widehat{\nabla J}_k^e(\w_{k, j-1}) \right\|}^4 | \w_{i-1} \right\} \notag \\
  {=}&\: \mu^4 p_k^4 K^4 \frac{L}{K} \E \left\{{\left\| \frac{e-1}{E_k} \frac{1}{e-1} \sum_{j=1}^{e-1}  \widehat{\nabla J}_k^e(\w_{k, j-1}) \right\|}^4 | \w_{i-1} \right\} \notag \\
  {=}&\: \mu^4 p_k^4 K^4 \frac{L}{K} \frac{(e-1)^4}{E_k^4} \E \left\{{\left\| \frac{1}{e-1} \sum_{j=1}^{e-1}  \widehat{\nabla J}_k^e(\w_{k, j-1}) \right\|}^4 | \w_{i-1} \right\} \notag \\
  \stackrel{(c)}{\le}&\: \mu^4 p_k^4 K^4  \frac{L}{K} \frac{(e-1)^3}{E_k^4} \sum_{j=1}^{e-1} \E \left\{ {\left\|  \widehat{\nabla J}_k^e(\w_{k, j-1}) \right\|}^4 | \w_{i-1} \right\} \notag \\
  \stackrel{(d)}{\le}&\: \mu^4 p_k^4 K^4  \frac{L}{K} \frac{1}{E_k} \sum_{j=1}^{e-1} \E \left\{ {\left\|  \widehat{\nabla J}_k^e(\w_{k, j-1}) \right\|}^4 | \w_{i-1} \right\} \notag \\
  \stackrel{(e)}{\le}&\: \mu^4 p_k^4 K^4  \frac{L}{K}  \frac{1}{E_k} \sum_{j=1}^{e-1} 8 \E \left\{{\left\|  {\nabla J}_k(\w_{k, j-1})  \right\|}^4| \w_{i-1} \right\} \notag \\
  &\: +\mu^4 p_k^4 K^4  \frac{L}{K}  \frac{1}{E_k} \sum_{j=1}^{e-1} 8 \E \left\{{\left\|  \s_{k, j-1}(\w_{k, j-1})\right\|}^4| \w_{i-1} \right\} \notag \\
  \stackrel{(f)}{\le}&\: \mu^4 p_k^4 K^4  \frac{L}{K} 8 U^4 +\mu^4 p_k^4 K^4  \frac{L}{K}  8 \left( \beta_k^4 U^4 + \sigma_k^4 \right) \notag \\
  =&\: \mu^4 p_k^4 K^4  \frac{L}{K} 8 \left(U^4+ \beta_k^4 U^4 + \sigma_k^4 \right)
\end{align}
where \( (a) \) and \( (b) \) follow from Bayes' theorem and \( \mathrm{Pr}\left\{\mathds{1}_{k, i} =1 \right\} = \frac{L}{K} \), \( (c) \) and \( (e) \) follows from~\eqref{eq:local_gradient_noise} and Jensen's inequality, \( (d) \) follows from \( e-1\le E_k \), and \( (f) \) follows from~\eqref{eq:bounded_fourth} and \eqref{eq:bounded_gradient} and the fact that \( e-1\le E_k \). Returning to~\eqref{eq:intermediate}, we have:
\begin{align}
  \E \left\{{\| \boldsymbol{d}_i\|}^4 | \w_{i-1} \right\} \stackrel{\eqref{eq:intermediate}}{\le}&\: \mu^4 \sum_{k=1}^K p_k^5 \frac{K^6}{L^2} \hat{\delta}^4 8 \left(U^4+ \beta_k^4 U^4 + \sigma_k^4 \right)
\end{align}

\end{document}